\documentclass[conference]{IEEEtran}
\usepackage{times}

\usepackage{amsmath} 
\usepackage{amssymb} 
\usepackage{amsthm} 
\usepackage{amsfonts} 

\usepackage[caption=false]{subfig}
\usepackage[pdftex]{graphicx}

\DeclareGraphicsExtensions{.pdf, .png, .jpg}

\usepackage{algorithm}
\usepackage[noend]{algpseudocode}
\algrenewcommand\algorithmicindent{0.7em}%

\usepackage{algpseudocode}

\newtheorem{problem}{Problem}
\newtheorem{theorem}{Theorem}
\newtheorem{remark}{Remark}
\newtheorem{lemma}{Lemma}

\newtheorem{example}{Example}

\DeclareMathOperator*{\argmin}{arg\,min} 

\newcommand{\prm}{PRM$^*$}

\usepackage[numbers]{natbib}
\usepackage{multicol}
\usepackage[bookmarks=true]{hyperref}

\begin{document}

\title{
Efficient Optimal Planning in  non-FIFO Time-Dependent Flow Fields
}




\author{James Ju Heon Lee$^1$, Chanyeol Yoo$^1$, Stuart Anstee$^2$ and Robert Fitch$^1$\\
\small $^1$University of Technology Sydney, Ultimo, NSW 2006, Australia\\
\small Email:  {\tt juheon.lee@student.uts.edu.au, \{chanyeol.yoo, robert.fitch\}@uts.edu.au}\\
\small $^2$Defence Science and Technology Group, Department of Defence, Australia\\
\small Email: {\tt stuart.anstee@dst.defence.gov.au}}




%

\maketitle

\begin{abstract}
We propose an algorithm for solving the time-dependent shortest path problem in flow fields where the FIFO (first-in-first-out) assumption is violated. This problem variant is important for autonomous vehicles in the ocean, for example, that cannot arbitrarily hover in a fixed position and that are strongly influenced by time-varying ocean currents. Although polynomial-time solutions are available for discrete-time problems, the continuous-time non-FIFO case is NP-hard with no known relevant special cases. Our main result is to show that this problem can be solved in polynomial time if the edge travel time functions are piecewise-constant, agreeing with existing worst-case bounds for FIFO problems with restricted slopes. We present a minimum-time algorithm for graphs that allows for paths with finite-length cycles, and then embed this algorithm within an asymptotically optimal sampling-based framework to find time-optimal paths in flows. The algorithm relies on an efficient data structure to represent and manipulate piecewise-constant functions and is straightforward to implement. We illustrate the behaviour of the algorithm in an example based on a common ocean vortex model in addition to simpler graph-based examples.
\end{abstract}

\IEEEpeerreviewmaketitle


\section{Introduction}\label{sec:introduction}
Many minimum-time planning problems in robotics inherently involve time-costs that are non-static. In terms of finding shortest paths on graphs, this means that edge traversal time is not a scalar value, but instead is a function that varies over time. The importance of developing shortest path algorithms for non-static travel times is well recognised, and somewhat surprisingly, has been studied for over 50 years~\cite{Cooke1966}. In comparison to static shortest path problems, progress in developing a theoretical understanding of the \emph{time-dependent shortest path}~(TDSP) problem has proved far more elusive. Our goal is to explore relatively recent theoretical results in an effort to develop practical algorithms for robotics applications.

Our main motivation is planning for robots and vehicles that are influenced by fluid flows, such as those in the ocean and the atmosphere~\cite{yoo2016online}. Planning in ocean currents is important for many applications such as oil and gas exploration~\cite{Russell-Cargill2018}, environmental monitoring~\cite{Rudnick2004} and defence~\cite{Johannsson2010}, with platforms such as underwater gliders, surface vessels, Wave Gliders, and profiling floats. Planning is critical when the maximum vehicle velocity is comparable to the prevailing current~\cite{lee2017energy,cadmus2019streamlines}; success of autonomous navigation is then directly tied to the ability to model~\cite{brian2019online} and exploit current predictions.

Known TDSP solution approaches remain difficult to apply due to the many subtle problem variants~\cite{Dean2004} whose complexity has only recently come to light. Perhaps as a consequence, many published algorithms have no stated performance bounds or, worse, make incorrect claims as noted in~\cite{Foschini2014}. One important property of TDSP problems over graphs is the FIFO (first-in-first-out) property, which essentially states that delaying departure time can never result in earlier arrival. Therefore, in FIFO problems, remaining at any given node is never beneficial. Waiting is critical for optimality in the non-FIFO case, although if arbitrary waiting is permitted then a non-FIFO problem can be transformed into an equivalent FIFO version~\cite{Dean2004}. A second important property lies in characterising the edge travel time function as either discrete time (one scalar value per unit time) or continuous time. The discrete time case (both FIFO and non-FIFO) is known to be polynomial in the number of edges and the length of the time horizon~\cite{Dean2004}. It is natural to exploit this discrete structure by searching the time-expanded graph using common algorithms such as A*. Execution time, however, can quickly become unwieldy for long time horizons, and this effect is especially problematic for slow moving vehicles in the ocean. The continuous time case, again for both FIFO and non-FIFO problems, is non-polynomial in the general case, even for piecewise-linear cost functions~\cite{Foschini2014}.

Whereas it is not possible to avoid the worst case bound in general, it is interesting to consider special cases. One restriction in FIFO problems is to limit the set of possible slopes for the pieces of the piecewise linear functions. This limitation is helpful because it allows for polynomial-time algorithms~\cite{Foschini2014}. 


In this paper, we consider a previously unidentified special case for non-FIFO problems where the edge travel time function is piecewise constant, and show that its worst-case bound is polynomial. We present an algorithm based on a data structure that supports efficient manipulations of the piecewise edge functions. We do not allow arbitrary waiting, which would not be possible in flow fields. Finally, our algorithm outputs solutions from every node and time to the destination, which means that the solution is in the form of a \emph{policy} that can be used for replanning, for example. The piecewise-constant assumption is reasonable in practice because ocean current estimates are typically provided in this form~\cite{chassignet2009us}, similar to upper atmosphere estimates~\cite{stein2015noaa}.

We then show how this solution for TDSP on graphs can be used for planning in flow fields. First, we argue that flow fields such as ocean currents are best described as \emph{asymptotically FIFO}, where the edge functions exhibit FIFO behaviour as edge lengths approach zero, yet practical problem instances are certainly non-FIFO. Then, we integrate the TDSP solution into a sampling based algorithm, using PRM* for illustration, and thus propose an integrated framework for time-dependent flows. We provide examples in 2D flows, but the algorithm can be used unchanged for 3D flows. We also show that the running time of the algorithm is polynomial in the number of nodes and size of edge functions.

This result is significant because it contributes a new theoretical result that provides a practical solution with known performance bounds. It also gives insight into how to model problems in practice.
Since the solution size heavily depends on the size of edge functions, it is desirable to minimise their size.
This modelling is different to how ocean current predictions would naively be used, because it encourages post-processing to merge pieces of individual edge functions that remain constant between successive estimates. 

\subsection{Related work}
Our work is related to the seminal paper by Foschini, Hershberger, and Suri~\cite{Foschini2014}, who showed that the computational complexity of FIFO problems with piecewise linear edge travel time functions is superpolynomial in the number of graph nodes. The authors further showed the existence of polynomial time special cases where the slopes of the travel time functions are restricted. We show that this bound also holds for finding minimum travel time paths in non-FIFO problems with piecewise constant edge functions, and present an algorithm with a straightforward implementation. The main distinction is that shortest paths in non-FIFO problems (that do not allow arbitrary waiting) may include cycles. Our analysis essentially bounds the length of cycles by relating the worst-case cycle time to properties of the edge functions.

A useful categorisation of problem variants along with complexity results is presented by Dean~\cite{Dean2004}. Interestingly, the tightest known worst case bounds for general FIFO and non-FIFO problems are both polynomial (in the number of graph edges and the time horizon) in the discrete time case, and are both superpolynomial in continuous time.

In discrete-time problems, optimal solutions may be found by searching the \emph{time-expanded graph}~\cite{Gunturi2011,Gunturi2015}, where nodes are duplicated per unit time. Dreyfus~\cite{Dreyfus1969} first applied Dijkstra's algorithm in this case. More recent work is based on A*~\cite{gliderPP_CTS_A*_2010,singleMulti_gliderPlanning_2011,gliderPP_iterative_ICRA2011} with performance improvements using adaptive discretisation~\cite{aSTS_Kularatne2018}, precomputed heuristics~\cite{Kontogiannis2016}, and bidirectional search~\cite{Demiryurek2011}.  Other approaches are based on time-aggregated graphs\cite{TimeAggregate_Betsy2008}, where edge functions are represented as time series.

Continuous-time problems are often motivated by path planning for underwater or surface vehicles in the ocean. A level set approach is presented by Lolla \emph{et al.}~\cite{Lolla2012,Lolla2014}. Recent work by Liu and Sukhatme~\cite{TVMDP_Liu2018} formulates the problem as a time-varying Markov decision process. Since these methods assume general edge functions, known complexity results suggest that their worst-case running time is non-polynomial. Our approach finds time-optimal paths in flow fields by embedding our efficient TDSP solution within a sampling-based framework that is asymptotically optimal.

\section{Background and Problem Formulation} \label{sec:problem_formulation}

\subsection{Time-dependent directed graph}

We consider a directed graph $G = (S,E)$ that consists of a finite set of states~$S$ and edges~$(s, s') \in E$ where~$s, s' \in S$. The set of immediately reachable states from state~$s$ is denoted as~$S_s \subseteq S$. A set of goal states is denoted as~$S_g \subset S$ where~$|S_g| \geq 1$. We restrict consideration to graphs in which goal states are reachable from the initial state.

We define an $(n+1)$-length path~$\Gamma$ within~$G$ as a sequence of states~$\Gamma = s_0 s_1 \cdots s_n$, where~$s_k \in S$ and $(s_{k}, s_{k+1}) \in E$ for all~$k$. The final state~$s_n$ is one of the goal states~$S_g$ while others are not. We denote~$\Gamma_{k}$ as the prefix of~$\Gamma$ up to the $k$-th state in the path (i.e., $\Gamma_{k} = s_0 s_1 \cdots s_k$).
Given an edge~$(s, s')$, we define \emph{edge travel time}~$C_{s s'}(t)$, or simply \emph{edge time}, as the time it takes to traverse from state~$s$ to~$s'$, having departed state~$s$ at time~$t$. Without loss of generality, path traversal begins no earlier than $t=0$, and edge time~$C_{s s'}(t)$ is $\infty$ for all~$t \leq 0$.

\subsection{Arrival and travel time}

\emph{Arrival time}~$a_{\Gamma}(t)$~\cite{Dean2004} is defined as the time of arrival at the final state~$s_n$ in path~$\Gamma$, having departed from state~$s$ at time~$t$. Formally,
\begin{equation}  \label{eqn:arrival_time}
    \begin{split}
        a_{\Gamma}(t) &= a_{s_{n-2}, s_{n-1}} ( a_{\Gamma_{n-2}}(t) ).
    \end{split}
\end{equation}

Similarly,  \emph{travel time}~$T_{\Gamma}(t)$ is defined as the time it takes to complete the path~$\Gamma$, having departed from the initial state~$s_0$ at time~$t$. Formally,
\begin{equation} \label{eqn:travel_time1}
    T_{\Gamma}(t) = a_{\Gamma}(t) - t
    .
\end{equation}
Intuitively, travel time begins at departure, whereas arrival time includes time spent awaiting the initial departure. Both arrival and travel times depend on the sequence of edge times, each of which depends on the individual arrival time at each edge.


\subsection{FIFO properties}

The TDSP problem is often solved for minimal arrival time assuming \emph{first-in-first-out} (FIFO) behaviour. A graph exhibits FIFO behaviour if it satisfies~\cite{Dean2004,Foschini2014} $t + C_{s s'}(t) \leq t' + C_{s s'}(t')$ for any edge~$(s, s') \in E$ and~$t \leq t'$.
Intuitively, the arrival time~$a_{\Gamma}(t)$ is non-decreasing with respect to departure time~$t$. 

Under the FIFO condition, optimal solutions exhibit the following properties~\cite{Dean2004}: 1) waiting at any state is not beneficial at any time, 2) optimal paths are acyclic (i.e., they do not revisit states), and 3) any subset of the optimal path is also a shortest path. 


\subsection{Problem formulation}

\begin{figure}
    \centering
    \subfloat[Initial state~$s_0$ and goal state~$s_1$ with time-dependent edge time function~$C$]{\includegraphics[width=0.45\columnwidth]{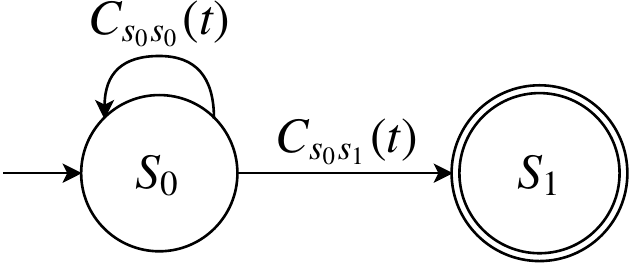} \label{fig:demo_graph}}    \\
    \subfloat[Edge time~$C_{s_0 s_0}(t)$]{\includegraphics[width=0.45\columnwidth]{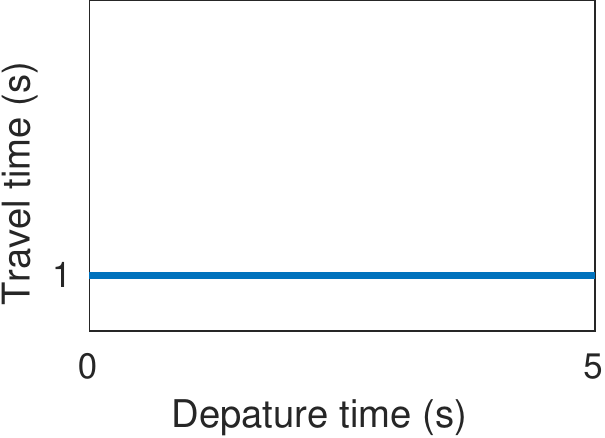} \label{fig:demo_c00}} 
    \subfloat[Edge time~$C_{s_0 s_1}(t)$]{\includegraphics[width=0.45\columnwidth]{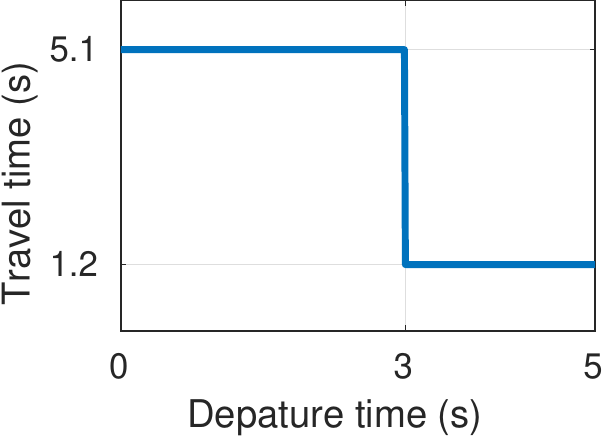} \label{fig:demo_c01}} 
    \caption{Two-state graph example with time-dependent edge times}
    \label{fig:demo}
\end{figure}

In this paper, the objective is to minimise the travel time given initial and goal states. In many practical robotics applications, travel time is more important than arrival time since the constraints such as energy and cost are often tightly coupled to travel time.
With the notation defined, the \emph{time-dependent shortest path} (TDSP) problem for travel time is defined as follows.
\begin{problem} [Minimum travel time problem] \label{prob:tdsp}
    Given a directed graph~$G$ with time-dependent edge time function~$C$, find an optimal path~$\Gamma^*$ and initial departure time~$t^*_0$ that minimises the travel time~$T$, such that
    \begin{equation}
        (\Gamma^*, t^*_0) = \argmin_{\Gamma, t} a_{\Gamma}(t) - t
        .
    \end{equation}
\end{problem}
\noindent We are also interested in solving for the minimum travel time for a set departure time~$t_0$ as follows.
\begin{problem} [Minimum travel time problem given initial departure time]
    \begin{equation}
        \Gamma^* = \argmin_{\Gamma} a_{\Gamma}(t_0) - t_0
        .
    \end{equation}
\end{problem}
\noindent We show later in this paper that the problems are equivalent under the proposed framework. Note that we do not consider waiting at an arbitrary node.

In this paper, we are interested in a non-FIFO graph where the arrival time may not be non-decreasing with respect to departure time.
\begin{example} [Non-FIFO graph] \label{example:nonfifo}
    Consider a graph with time-dependent edge time function as shown in Fig.~\ref{fig:demo}. The edge time~$C_{s_0 s_1}$ is 5.1 seconds for the first 3.5 seconds and then reduces to 1.2 seconds. The self-transition edge time~$C_{s_0 s_0}$ is 1 second for all departure time. 
\end{example}
\noindent Clearly, the graph does not exhibit FIFO behaviour.
Transitioning immediately to goal state~$s_1$ from~$s_0$ takes longer than self-transitioning at state~$s_0$ for a few times before arriving at goal state~$s_1$. 
Note that self-transitioning is different to waiting; the former is defined as an edge transition allowed by a graph, whereas the latter is an arbitrary hold duration which is not defined in a graph.


\section{Piecewise-Constant Representation} \label{sec:evaluation}

We present travel and edge time functions in the form of \emph{piecewise-constant functions} (PF)~\cite{PPF_Yoo2012,yoo2014provably}. We illustrate the form and the relevant operations required to understand the PF-based value iteration where the solution is given as a policy rather than a path.

\subsection{Definition of piecewise-constant function (PF)} \label{subsec:piecewise-constant}
A piecewise-constant function~$f : \mathbb{R} \rightarrow \mathbb{R}$ is defined as a sequence of subdomain and constant pairs, where each subdomain is a time interval in which the value is constant.
We define a piecewise-constant function~$f(t)$ on subdomains~$p^f_i$ indexed backwards in time as follows
\begin{equation} \label{eqn:pf}
    f(t) = \begin{cases}
        v_1^f,& \text{if}~t > p_1^f\\
        \vdots\\
        v_k^f,& \text{if}~p_{k-1}^f \geq t > p_k^f\\
        \vdots\\
        v_n^f,& \text{if}~p_{n-1}^f \geq t > p_n^f\\
        \infty,& \text{else}
        \end{cases} \equiv \begin{cases}
        v_1^f,& \text{if}~ t > p_1^f\\
        \vdots\\
        v_k^f,& \text{ef}~t > p_k^f\\
        \vdots\\
        v_n^f,& \text{ef}~t > p_n^f\\
        \infty,& \text{else}
    \end{cases}
\end{equation}
where~$k\in\mathbb{N}$ and~$p_{k+1}^f < p_k^f\, \forall k \in \mathbb{N}$. We use a short form~`ef' for `else if'. 
Operations over PF are defined in~\cite{PPF_Yoo2012}.
We define an additional operation called~\emph{recursion}. Suppose we have two PFs~$f(t)$ and~$g(t)$. A recursive PF~$f(t + g(t))$ can be derived using conditioning and merging operators such that
\begin{equation}
    \begin{split}
        f(t + g(t)) = &(f(t + p_1^g)~\ominus~p_1^g) \oplus \cdots \oplus(f(t + p_n^g)~\ominus~p_n^g)
    \end{split}
    .
\end{equation}

\subsection{TDSP with piecewise-constant functions}\label{sec:belmans_equation}

We represent the travel time~$T_{\Gamma}(t)$ and edge time~$C_{s s'}$ using PF, as defined in~(\ref{eqn:pf}).
The travel time function in~(\ref{eqn:travel_time1}) can be re-written in an iterative form where~$T^k_s(t)$ is the travel time for $k$ edge transitions starting at state~$s$. Formally,
\begin{equation} \begin{split}
    T^0_{s_0}(t) =& 0    \\
    T^1_{s_0}(t) =& C_{s_0 s_1}(t) \\
    T^2_{s_0}(t) =& C_{s_0 s_1}(t) + C_{s_1 s_2}(t + C_{s_0 s_1}(t))   \\
    T^3_{s_0}(t) =& C_{s_0 s_1}(t) + C_{s_1 s_2}(t + C_{s_0 s_1}(t))\\
         &+ C_{s_2 s_3}(C_{s_0 s_1}(t) + C_{s_1 s_2}(t + C_{s_0 s_1}(t)))   \\
        =& C_{s_0 s_1}(t) + T^2_{s_1}(t + C_{s_0 s_1}(t))
    .
\end{split} \end{equation}
The travel time function can be written recursively as
\begin{equation} \label{eqn:travel_time2}
    T^{k+1}_{s}(t) = C_{s s'}(t) + T^k_{s'}(t + C_{s s'}(t))
    .
\end{equation}
We denote by~$T^*_s(t)$ the converged travel time, where~$T^{k+1}_s(t) = T^{k}_s(t)$ for all~$t \in \mathbb{R}$ and some finite~$k \in \mathbb{Z}$.

We give the solution to the general TDSP problem as a \emph{travel policy}~$\pi_s(t)$ represented as a PF for each state, allowing for time-dependent transitions to other states. Formally,
\begin{equation}
    \pi_s(t) = 
    \begin{cases}
        s^1_s \in S_s &   \text{if}~t > p^1_s   \\
        s^2_s \in S_s &   \text{ef}~t > p^2_s   \\
        \vdots  &   \vdots \\
        s^{n_{\pi_s}}_s \in S_s &   \text{ef}~t > 0 \\
        \emptyset   &   \text{else}
    \end{cases}
    ,
\end{equation}
where~$s^i_s$ is the state visited during the $i$-th subdomain. 
As noted, the edge travel time function~$C_{ss'}(t)$ for edge~$(s, s')$ at departure time~$t$ is a PF, where the last subdomain is~$0$ and the edge time is~$\infty$ for `else' subdomain (i.e., $\forall t \in (\infty, 0]$).

Given a travel policy~$\pi$, the PF-based travel time function is written as
\begin{equation} \label{eqn:eval_travel}
    T^{k+1}_{s}(t) = 
    \begin{cases}
        C_{s s^1_s}(t) + T^k_{s^1_s}(t + C_{s s^1_s}(t)) & \text{if}~t > p^1_s \\
        C_{s s^2_s}(t) + T^k_{s^2_s}(t + C_{s s^2_s}(t)) &   \text{ef}~t > p^2_s   \\
        \vdots & \vdots\\
        C_{s s^{n_{\pi_s}}_s}(t) + T^k_{s^{n_{\pi_s}}_s}(t + C_{s s^{n_{\pi_s}}_s}(t)) &   \text{ef}~t > 0  \\
        \infty &   \text{else}
    \end{cases}
    ,
\end{equation}
where~$p^k_s$ is the beginning of the $k$-th subdomain in policy~$\pi_s$.
The travel time for the goal states~$s \in S_g$ at any iteration~$k$ is~$0$ for all~$t \in \mathbb{R}$.
For simplicity, we use a short form omitting the~`$else$' case, i.e., $t \in (-\infty, 0]$.

\begin{example} [Calculation example] \label{example:pf_calc}
    We illustrate the calculation of travel time using PF-based travel policy for the example in Fig.~\ref{fig:demo_graph}. The edge times in PF form are
    \begin{equation} \begin{split}
        C_{s_0 s_0} &= 
        \begin{cases}
            1.6 &   \text{if}~t > 0
        \end{cases}~\text{and}~
        C_{s_0 s_1} = 
        \begin{cases}
            1.2 &   \text{if}~t > 3.5   \\
            5.1 &   \text{ef}~t > 0
        \end{cases}
    \end{split} \end{equation}
    Let the travel policy for state~$s_0$ be
    \begin{equation} \label{eqn:toy_policy}
        \pi_{s_0}(t) = \begin{cases}
            s_1 &   \text{if}~t > 3 \\
            s_0 &   \text{ef}~t > 0
        \end{cases}
        .
    \end{equation}
    The travel time for~$s_0$ after $k+1$ edge transitions is then
    \begin{equation} \begin{split}
        T^{k+1}_{s_0}(t) &= 
        \begin{cases}
            C_{s_0 s_1}(t) + T^k_{s_1}(t + C_{s_0 s_1}(t)) &   \text{if}~t > 3  \\
            C_{s_0 s_0}(t) + T^k_{s_0}(t + C_{s_0 s_0}(t)) &   \text{ef}~t > 0
        \end{cases} \\
        &=
        \begin{cases}
            \begin{cases}
                1.2 &   \text{if}~t>3.5   \\
                5.1 &   \text{else}
            \end{cases} &   \text{if}~t > 3  \\
            1.6 + T^k_{s_0}(t + 1.6) &   \text{ef}~t > 0
        \end{cases} \\
        &=
        \begin{cases}
            1.2 &   \text{if}~t > 3.5   \\
            5.1 &   \text{ef}~t > 3 \\
            1.6 + T^k_{s_0}(t+1.6)  &   \text{ef}~t > 0
        \end{cases}
        ,
    \end{split} \end{equation}
    The detailed operations over piecewise functions including the~\emph{conditioning} operation are presented in~\cite{PPF_Yoo2012}. Starting from~$T^0_s(t) = 0$, we iteratively calculate the travel time until it converges. The corresponding travel time subject to the heuristic policy (\ref{eqn:toy_policy}), which is not optimal, is
    \begin{equation}
        T_s(t) = 
        \begin{cases}
        	1.2	&	\text{if}~t > 3.5	\\
        	5.1	&	\text{ef}~t > 3.0	\\
        	2.8	&	\text{ef}~t > 1.9	\\
        	6.7	&	\text{ef}~t > 1.4	\\
        	4.4	&	\text{ef}~t > 0.3	\\
        	8.3	&	\text{ef}~t > 0
        \end{cases}
        .
    \end{equation}
    Figure~\ref{fig:trav_and_arrv_eval} illustrates the solution; departing at~$t = 0$, the policy causes self-transit twice at~$s_0$ before moving to~$s_1$. If the initial departure time is greater than $3s$, the policy is to transit to~$s_1$ straight away.
\end{example}

\begin{figure}
    \centering
    \subfloat[Travel time]{\includegraphics[width=0.5\columnwidth]{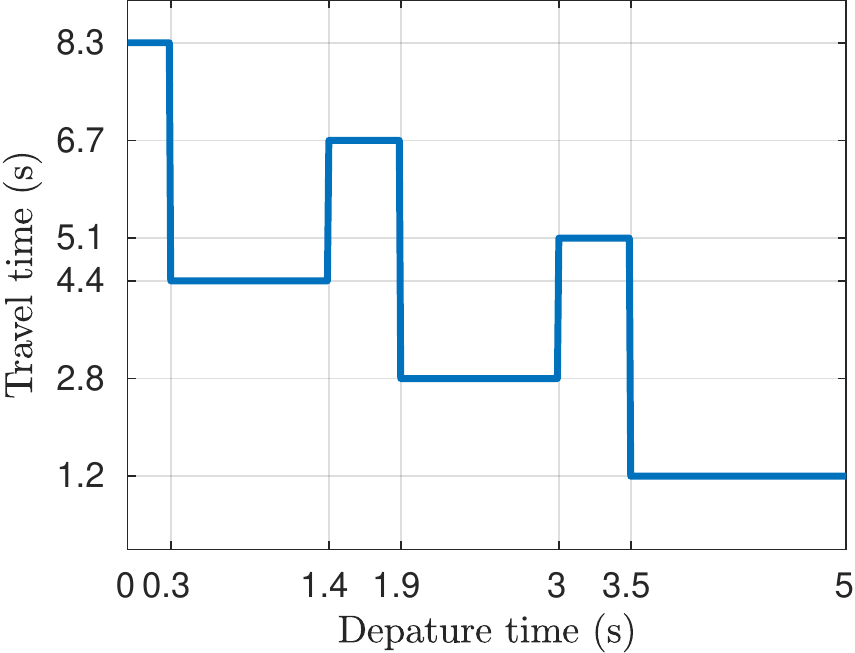} \label{fig:travTime_eval}}
    \subfloat[Arrival Time]{\includegraphics[width=0.5\columnwidth]{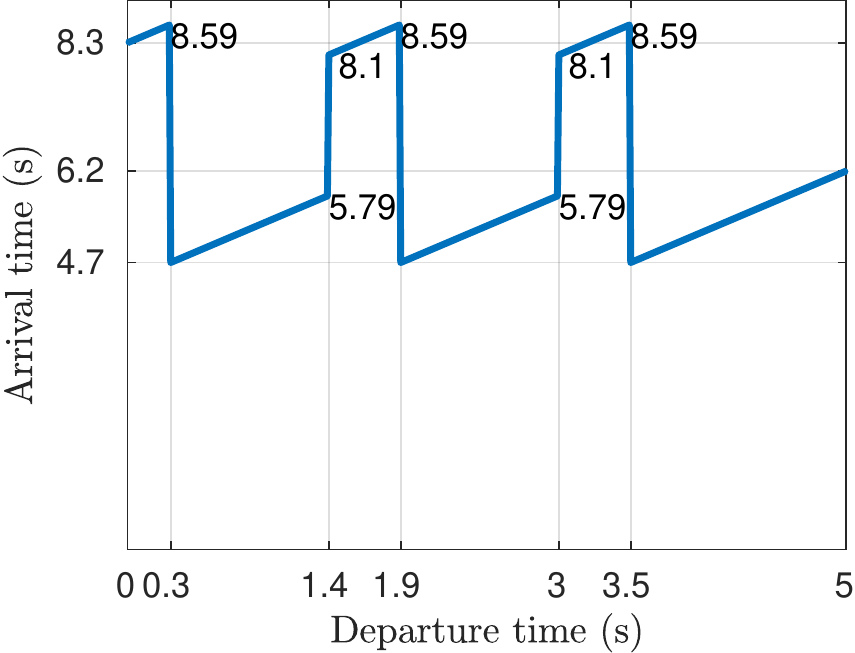} \label{fig:ArrvTime_eval}}
    \caption{Travel and arrival time for travel policy in Example~\ref{example:pf_calc}}
    \label{fig:trav_and_arrv_eval}
\end{figure}



\section{Optimal Travel Time Policy for Non-FIFO TDSP Problems} \label{sec:synthesis}

Given a directed graph~$G$ and time-dependent edge time function~$C_{s s'}(t)$, the optimal travel policy~$\pi^*_s$ to reach the goal state~$s_g \in S_g$ from state~$s$ can be expressed as
\begin{equation} \label{eqn:opt}
    \pi^*_{s}(t) = 
    \begin{cases}
        \vdots & \vdots\\
        \displaystyle \argmin_{s' \in S_s} C_{s s'}(t) + T^*_{s'}(t + C_{s s'}(t)) &   \text{ef}~t > p^{i*}_s   \\
        \vdots & \vdots
    \end{cases}
    .
\end{equation}
In principle, the optimal solution of~(\ref{eqn:opt}) at state~$s$ could be computed iteratively, by finding the optimal next state~$s' \in S_s$ for each time~$t >0$. 
Such an exhaustive approach is impossible, since there are infinitely many subdomains to consider.

We avoid this problem by finding a finite set of subdomains iteratively. We prove later that such set exists for an optimal solution.
Let~$T^{k+1}_{s s'}$ be an \emph{immediate travel time function} in which the transition from state~$s$ to~$s'$ occurs over all time~$t$ after $k$ edge transitions. Formally,
\begin{equation} \label{eqn:immediate_travel}
    T^{k+1}_{s s'} = 
    \begin{cases}
        C_{s s'}(t) + T^{k}_{s'}(t + C_{s s'}(t)) &   \text{if}~t > 0
    \end{cases}
    .
\end{equation}
Let $P^{k}_s$ be the set of subdomains in travel time function~$T^{k}_s$. Such a set for optimal travel time~$T^{k+1*}_s$ is
\begin{equation} \label{eqn:subdomains}
    P^{k+1*}_s = \{ 0 \} \cup \bigcup_{s' \in S_s} P^{k+1*}_{s s'}
    ,
\end{equation}
where subdomain set~$P^{k+1*}_{s s'}$ is from immediate travel function~$T^{k+1*}_{s s'}$ and~$P^0_s = \{0\}$.
The optimisation problem in~(\ref{eqn:opt}) is then solved over a finite set of subdomains~$P^{k+1*}_s$ where~$p^{i*}_s \in P^{k+1*}_s$ in~(\ref{eqn:opt}).
The pseudocode is presented in Alg.~\ref{algo:pseudoCode_optimal_policy}.

\begin{example} [Optimal policy example] \label{example:synthesis}
    The optimal travel time with respect to the graph in Fig.~\ref{fig:demo_graph} is
    \begin{equation}
        \pi^*_{s_0}(t) =
        \begin{cases}
        	s_1	&	\text{if}~t > 3.5	\\
        	s_0 &   \text{ef}~t > 0.3   \\
        	s_1	&	\text{ef}~t > 0
        \end{cases}
        ,~
        T^*_{s_0}(t) = 
        \begin{cases}
        	1.2	&	\text{if}~t > 3.5	\\
        	2.8	&	\text{ef}~t > 1.9	\\
        	4.4	&	\text{ef}~t > 0.3	\\
        	5.1	&	\text{ef}~t > 0
        \end{cases}
    \end{equation}
    The solution converged after 5 iterations.
\end{example}
\noindent Intuitively, the optimal travel policy for Fig.~\ref{fig:demo} should be to self-transit until the time is~$3.5s$, since edge time~$C_{s_0 s_1}$ is high. However, the optimal policy is interestingly non-intuitive. Suppose the policy is to self-transit while ~$0 < t < 3.5$, then the induced path is~$s_0 s_0 s_0 s_1$. The travel time after two self-transitions is~$3.2s$ which causes transition to state~$s_1$ $0.3s$ before the cheap cost is available. Therefore it is less time consuming to transit directly than self-transit. Notably, self-transiting becomes more efficient just after~$t_0 > 0.3$ since the travel time after two edge transitions is enough to transit to~$s_1$ when the cost is cheap.

\begin{algorithm}[t]
	\begin{algorithmic}
	    \State \textbf{Inputs:} Directed graph~$G = (S, E)$, time-dependent edge time function~$C$ and number of edge transitions~$K$
	    \State \textbf{Outputs:} Optimal travel policy~$\pi^{k*}$ and travel time~$T^{k*}$
	    
		\State $T^0_{s} \leftarrow \begin{cases}
                                0,& \text{if } t > 0
                          \end{cases},
                          \forall  s \in S$
		\For {$k \leftarrow 1$~\textbf{to}~$K$}
		    \ForAll{$s\in S\setminus{s_n}$}
		        \State $T^{k+1}_{s s'} \leftarrow \begin{cases} \infty & \text{if}~t>0 \end{cases}, \forall s' \in S$ \Comment{(\ref{eqn:immediate_travel})}
		        \State $P^{k+1*}_s \leftarrow \{0\}$ \Comment{(\ref{eqn:subdomains})}
		        \ForAll {$s' \in S_s$}
    		        \State $T^{k+1}_{s s'} \leftarrow
                        \begin{cases}
                            C_{s s'}(t) + T^{k}_{s'}(t + C_{s s'}(t)) &   \text{if}~t > 0
                        \end{cases}$ \Comment{(\ref{eqn:immediate_travel})}
                    \State $P^{k+1*}_s \leftarrow P^{k+1*}_s \cup P^{k+1*}_{ss'}$ \Comment{(\ref{eqn:subdomains})}
                \EndFor
                \State $T^{k+1*}_{s} \leftarrow 
                \begin{cases}
                    \vdots  &   \vdots  \\
                    \displaystyle \min_{s' \in S_s} T^{k+1}_{s s'}    &   \text{ef}~t > p^{i*}_s    \\
                    \vdots& \vdots
                \end{cases}$, $\forall p^{i*}_s \in P^{k+1*}_{s}$
                \State $\pi^{k+1*}_{s} \leftarrow 
                \begin{cases}
                    \vdots  &   \vdots  \\
                    \displaystyle \argmin_{s' \in S_s} T^{k+1}_{s s'}    &   \text{ef}~t > p^{i*}_s    \\
                    \vdots& \vdots
                \end{cases}$, $\forall p^{i*}_s \in P^{k+1*}_{s}$
                
		    \EndFor
		    
		    \If{$P^{k+1*}_s \equiv P^{k*}_s$ and~$T^{k+1*}_s(t) \equiv T^{k*}_s(t)$, $\forall t \in P^{k*}_s, s \in S$ }
		        \State \textbf{break}
		    \EndIf
		\EndFor
		\Return $\pi^{k}_{s}(t)$ and $T^{*}_{s}(t)$
		\caption{Solving for optimal policy~$\pi^{k*}_{s}(t)$}
		\label{algo:pseudoCode_optimal_policy}
	\end{algorithmic}
\end{algorithm}

\section{Analysis}\label{sec:analysis}

Without loss of generality, we assume~$S_s = S$ for all states~$s \in S \setminus S_g$ (i.e., all states are immediately reachable from any state). 
Given~(\ref{eqn:immediate_travel}) and~(\ref{eqn:subdomains}), the set of subdomains can be written as
\begin{equation}
    P^{k+1}_s = \{0\} \cup \left\{ p \in \bigcup_{s' \in S_s} C_{s s'} \cup \left( P^k_{s'} - C_{s s'} \right) \mid p \geq 0 \right\}
    ,
\end{equation}
where~$A - B = \{a - b \mid \forall a \in A~\text{and}~b \in B\}$ for two sets~$A$ and~$B$. We slightly abuse notation for~$C_{ss'}$ to denote the set of subdomains and constants in the corresponding edge time function. Using a short form~$\frac{A}{0} = \{a \in A \mid a \geq 0\}$, we have
\begin{equation} \begin{split} \label{eqn:analysis_P1}
    P^{k+1}_s = &\{0\} \cup \bigcup_{s'} C_{ss'} \cup \bigcup_{s'} \frac{P^k_{s'} - C_{s s'}}{0}    \\
    = &\{0\} \cup \bigcup_{s'} C_{ss'} \cup \bigcup_{s'} \bigcup_{s''} \frac{C_{s's''} - C_{ss'}}{0}   \\
    &\cup \bigcup_{s'} \bigcup_{s''} \frac{ \frac{P^{k-1}_{s'} - C_{s's''}}{0}  - C_{ss'}}{ 0 }
    .
\end{split} \end{equation}
Since all subdomain sets for edge time function are non-negative by definition, the last term can be written in a form
\begin{equation} \begin{split}
    \frac{\frac{A}{0} - \frac{B}{0}}{0} = \{ a - b \mid a \in A, b \in B, a \geq 0, b \geq 0, a-b \geq 0 \}
    .
\end{split} \end{equation}
Since~$B$ (i.e., $C_{ss'}$) is non-negative, 
\begin{equation}
    \frac{A - B}{0} = \{a - b \mid a \in A, b \in B, a - b \geq 0~\text{and}~b \geq 0\}
    .
\end{equation}
Intuitively, if~$a - b \geq 0$ and~$b \geq 0$, then~$a \geq 0$. Therefore
\begin{equation}
    \frac{\frac{A}{0} - B}{0} \equiv \frac{A - B}{0}
    ,
\end{equation}
if~$B$ is a non-negative set. Then the last term in~(\ref{eqn:analysis_P1}) becomes
\begin{equation}
    \frac{ \frac{P^{k-1} - C_{s's''}}{0}  - C_{ss'}}{ 0 } = \frac{P^{k-1}_{s'} - C_{ss'} - C_{s's''}}{0}
    .
\end{equation}
Therefore, the set of subdomains can be recursively written as
\begin{equation} \begin{split} \label{eqn:analysis_P}
    P^{k+1}_s
    = &\{0\} \cup \bigcup_{s'} C_{ss'} \cup \bigcup_{s'} \bigcup_{s''} \frac{C_{s's''} - C_{ss'}}{0}   \\
    &\cup \cdots \cup 
    \bigcup_{s'} \cdots \bigcup_{s^k} \frac{ C_{s^{k-1}s^k} - \cdots - C_{ss'} }{0}
    .
\end{split} \end{equation}

\begin{lemma} [Finite edge transitions for convergence given infinite length path] \label{lemma:bound}
    Given an arbitrary and infinite length path~$\Gamma$, the set of subdomains converges to a unique and finite set within a finite number of edge transitions.
\end{lemma}
\begin{proof}
    All subdomains in any edge time function are non-negative by definition. Therefore subdomains in~(\ref{eqn:analysis_P}) monotonically decrease as the number of edge transitions increase. Since all subdomains in travel time function are non-negative, there exists a finite maximum number of edge transitions~$K_{\max}$ before convergence.
\end{proof}

\begin{lemma} [Convergence in finite edge transitions] \label{lemma:conv_kmax}
    Given an arbitrary and infinite length path~$\Gamma$, the worst case number of edge transitions before convergence in subdomain set is
    \begin{equation}
        K_{\max} = \mathrm{ceil} \left( \frac{ \max C }{ \min C } \right)
        ,
    \end{equation}
    where~$C = \bigcup_{s \in S} \bigcup_{s' \in S} C_{ss'} \setminus \{0\}$.
\end{lemma}
\begin{proof}
    The worst-case number of edge transition~$K_{\max}$ is the last edge transition before the subdomain set in $K_{\max}$-th term becomes empty in~(\ref{eqn:analysis_P}) since subdomains are non-negative. The longest such term is~$\max C - \sum_{k=1}^{K_{\max}} \min C$.
\end{proof}

\begin{remark} [Convergence and cyclic paths] \label{remark:cyclic}
    The length of an optimal path that may include cycles is bounded by a finite number of edge transitions found in Lemma~\ref{lemma:conv_kmax}.
\end{remark}


\begin{theorem} [Convergence in optimal algorithm]
    The optimal algorithm in Alg.~\ref{algo:pseudoCode_optimal_policy} converges in a finite time~$K_{\max}$ as shown in Lemma~\ref{lemma:conv_kmax}.
\end{theorem}
\begin{proof}
    By Lemma~\ref{lemma:bound} and~\ref{lemma:conv_kmax}, there exists a unique and finite set of subdomains for travel time functions. Since the optimisation problem is to find optimal travel policy for each subdomain, the problem is then solved in a finite number of iterations.
\end{proof}

\begin{theorem} [Time complexity] \label{theorem:complexity}
    The overall time complexity for Alg.~\ref{algo:pseudoCode_optimal_policy} is~$\mathcal{O}((|S| \cdot |C_m|)^{k+2})$, where~$|S|$ is the number of states in graph~$G$, $|C_m|$ is the maximum number of subdomains over a set of edge time functions and~$k$ is the number of edge transitions.
\end{theorem}
\begin{proof}
    By~(\ref{eqn:eval_travel}), the overall complexity is related to the number of subdomains, which is bounded by
    \begin{equation} \begin{split}
        |P^{k+1}_s| 
        &= \left| \{0\} \cup \bigcup_{s'} C_{ss'} \cup \bigcup_{s'} \frac{P^{k}_{s'} - C_{s s'}}{0} \right|   \\
        &\leq 1 + (|S| \cdot |C_m|) + (|S| \cdot |C_m|) \cdot |P^{k}_{s'}|    \\
        &\leq 1 + 2 (|S| \cdot |C_m|) + 2 (|S| \cdot |C_m|)^2 \\
        &+ \cdots + 2 (|S| \cdot |C_m|)^{k+1}   \\
        &= \mathcal{O}((|S| \cdot |C_m|)^{k+2})
        ,
    \end{split} \end{equation}    
    where~$C_m$ is the set of subdomains and constants with the maximum cardinality over all edges~$e \in E$. In the worst case, we find the optimal policy for each subdomain using value iteration in Alg.~\ref{algo:pseudoCode_optimal_policy}. Since the time complexity for solving such value iteration is polynomial in number of states, the overall time complexity for the proposed algorithm is~$\mathcal{O}((|S| \cdot |C_m|)^{k+2})$.
\end{proof}

\begin{remark} [Time-static reduction]
    From Theorem~\ref{theorem:complexity}, the time complexity for time-static edge functions is reduced to~$\mathcal{O}(|S|^2)$ (i.e., $|C_m|=1$ and~$K=0$) which agrees with the complexity of static shortest path problems.
\end{remark}

\begin{remark} [Time complexity in practice] \label{remark:complexity_practice}
    The time complexity in Theorem~\ref{theorem:complexity} is based on three worst-case conditions: 1) all states are connected to all the others, 2) the maximum number of iterations depends on the smallest edge subdomain and 3) no overlapping subdomains. 
\end{remark}
\noindent The worst-case conditions in Remark~\ref{remark:complexity_practice} occur rarely in practice, particularly Condition 3. When subdomains in a set overlap, they merge and form a much smaller set. Therefore the set does not grow indefinitely until convergence. 

By Theorem~\ref{theorem:complexity}, the complexity heavily depends on the number of edge transitions (i.e., iterations). Since the maximum number of edge transitions depends on the subdomains among all edge functions as shown in Lemma~\ref{lemma:conv_kmax}, the running time of the algorithm can be improved significantly by adaptively pruning early rapid changes to increase the denominator (i.e., $\min C$). Furthermore, we can also reduce the size of the edge functions by merging consecutive constants that are similar (i.e., reducing~$|C_m|$).

\begin{figure}
    \centering
    \includegraphics[width=0.6\columnwidth]{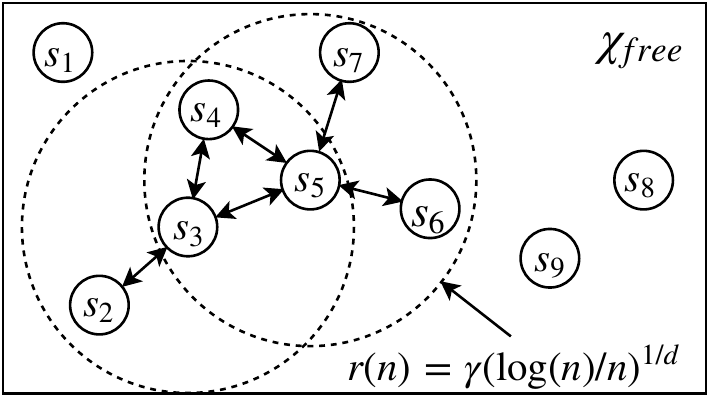}
    \caption{A partial construction of a PRM* graph. Note the formation of a cyclic graph in the radius intersection.}
    \label{fig:prmstar}
\end{figure}

\begin{figure*}[t]
    \centering
    \subfloat[At~$s_5$ and~$t=2$]{\includegraphics[width=0.35\columnwidth]{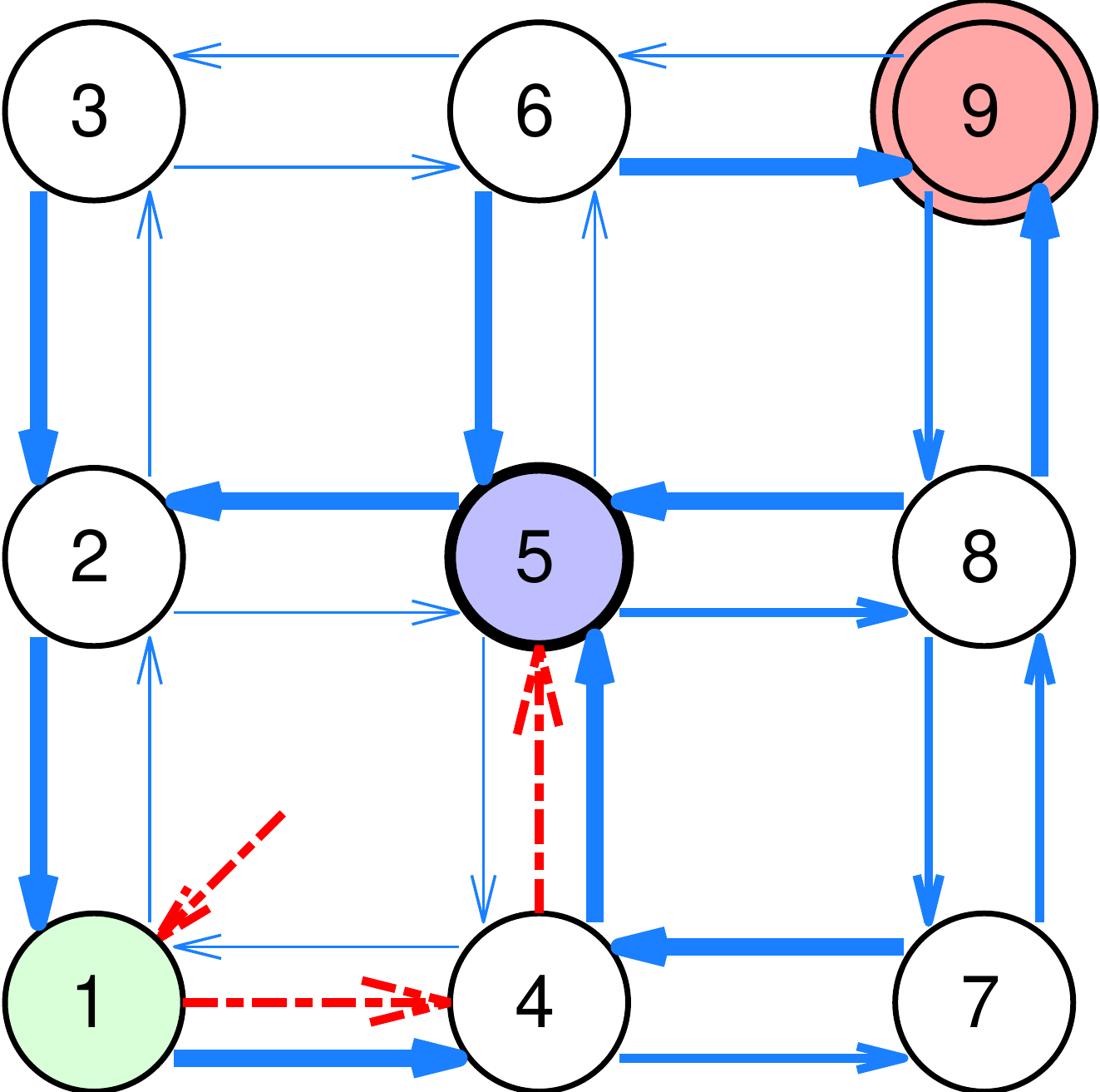} \label{subfig:GridCase_frame1}}
    \subfloat[At~$s_2$ and~$t=7$]{\includegraphics[width=0.35\columnwidth]{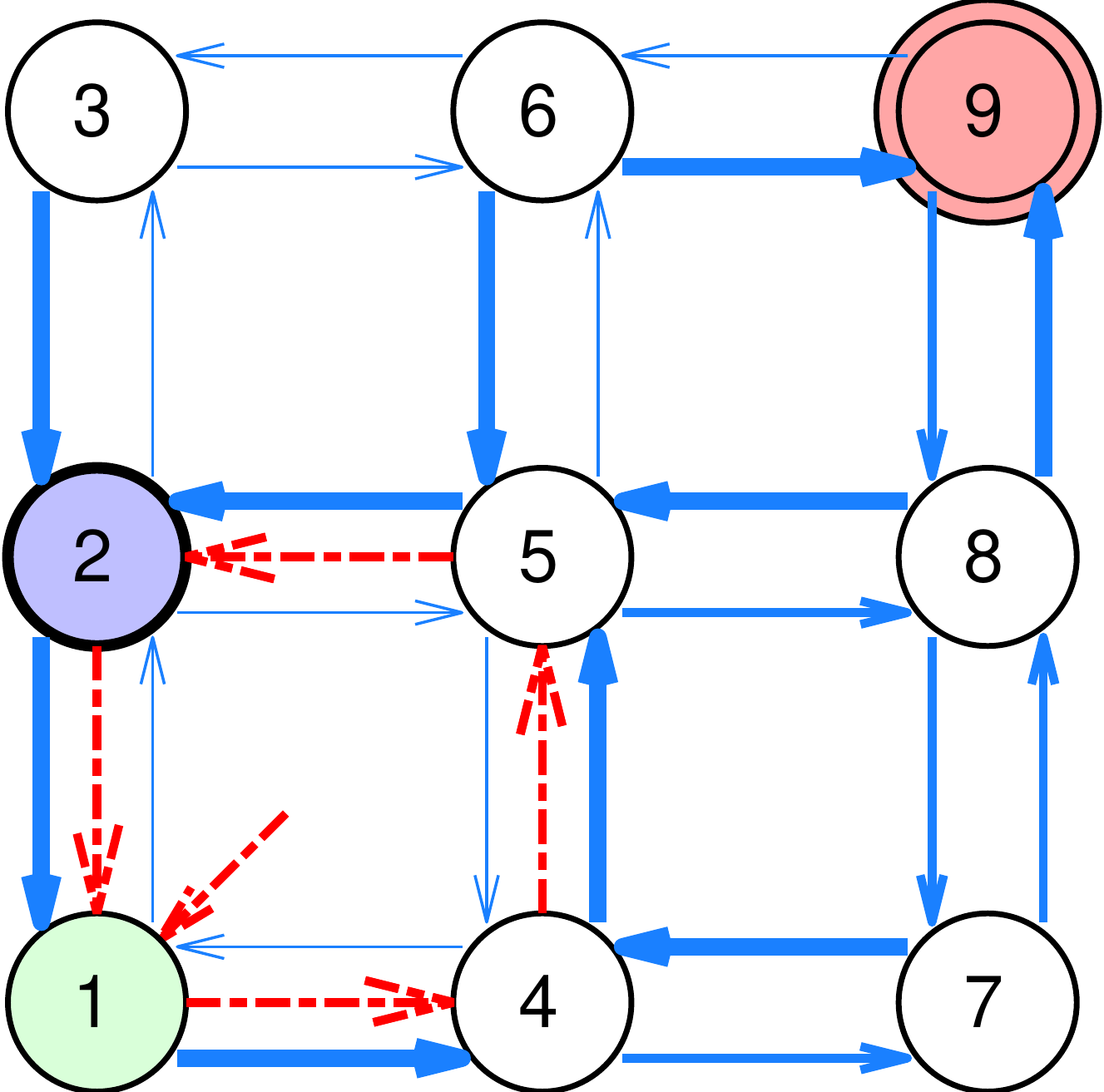} \label{subfig:GridCase_frame2}}
    \subfloat[At~$s_5$ and~$t=10$]{\includegraphics[width=0.35\columnwidth]{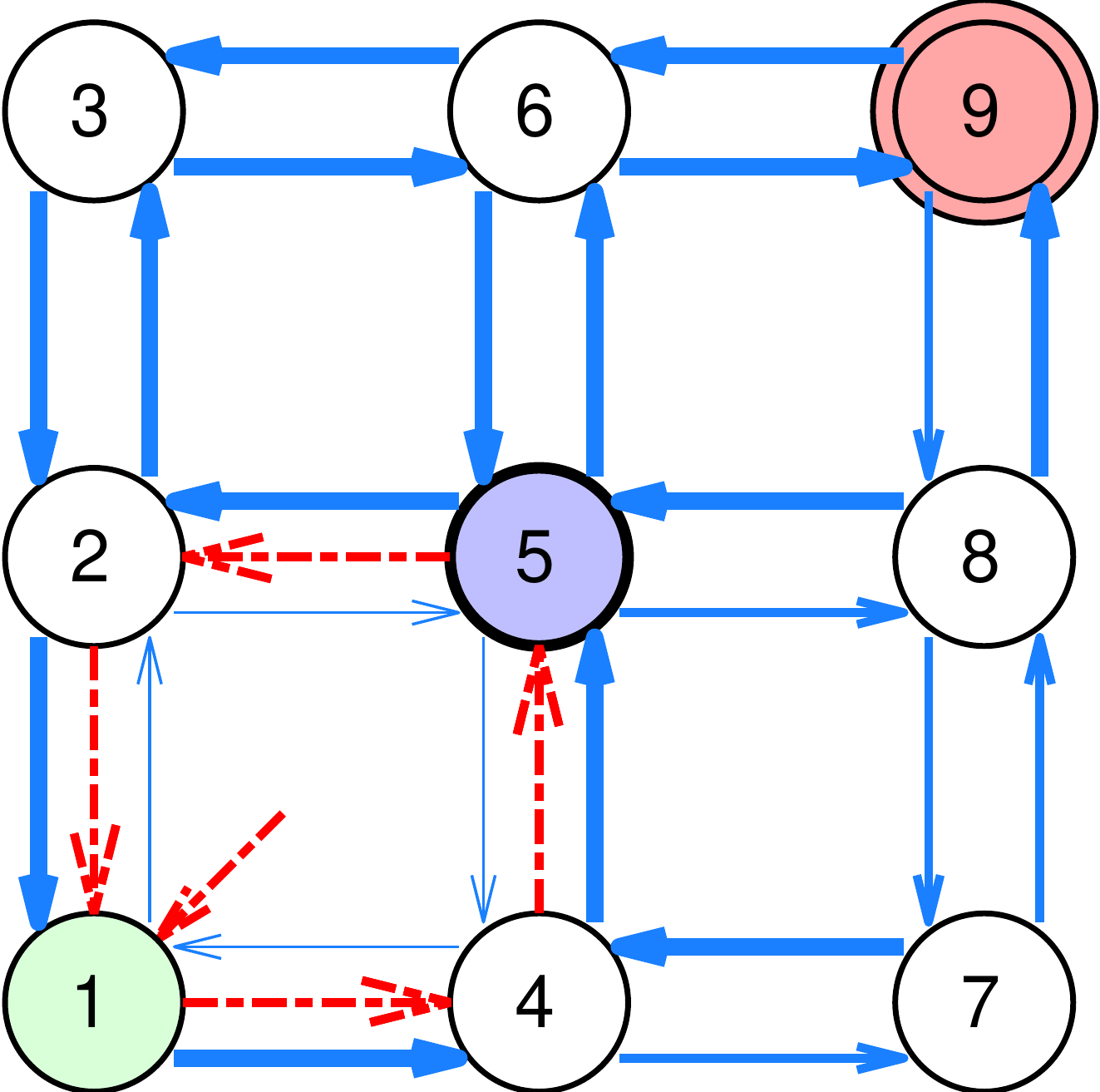} \label{subfig:GridCase_frame3}}
    \subfloat[At~$s_9$ and~$t=12$]{\includegraphics[width=0.35\columnwidth]{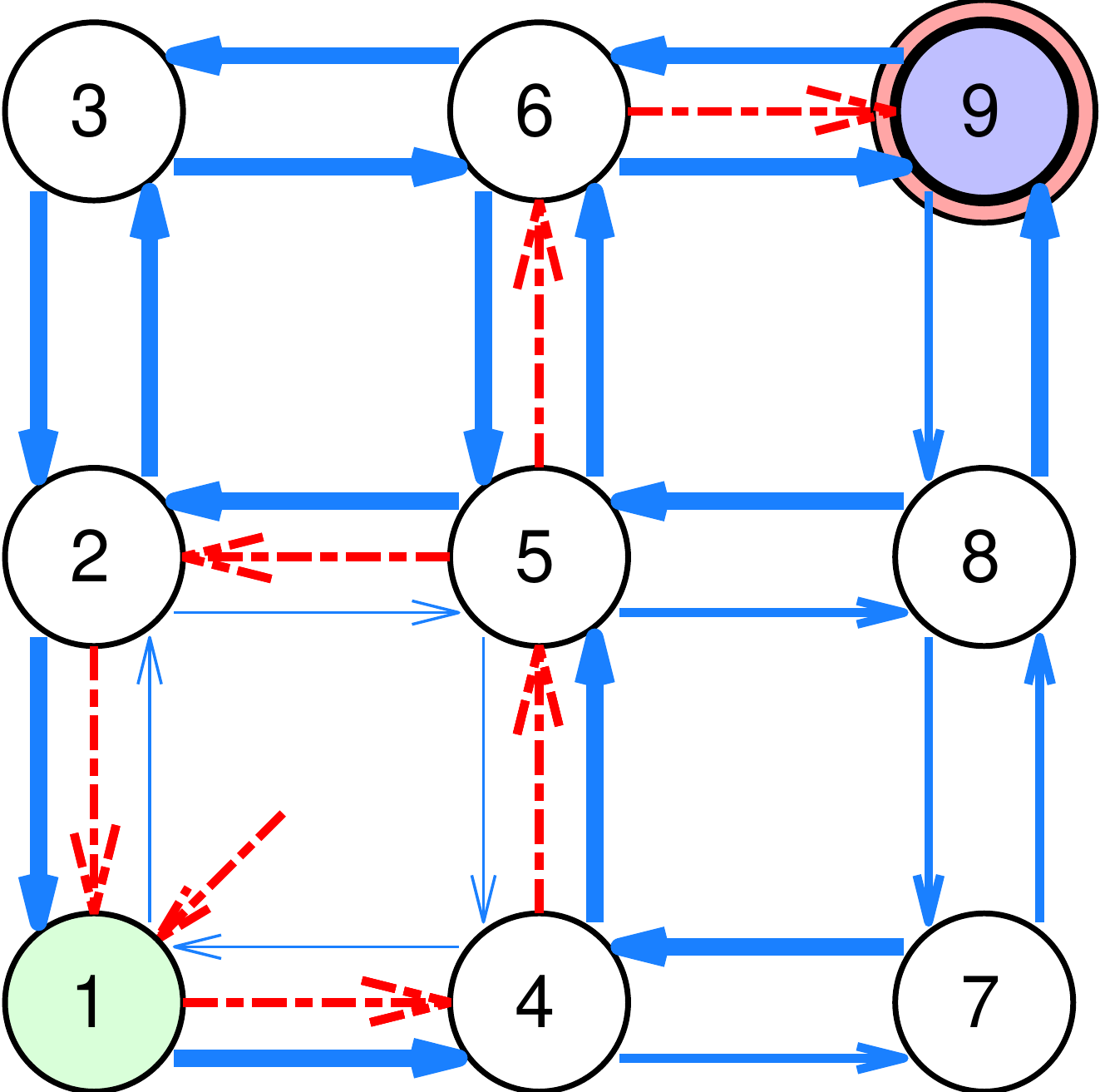} \label{subfig:GridCase_frame4}}
    \\
    \subfloat[Optimal travel time at~$s_1$]{\includegraphics[width=0.4\columnwidth]{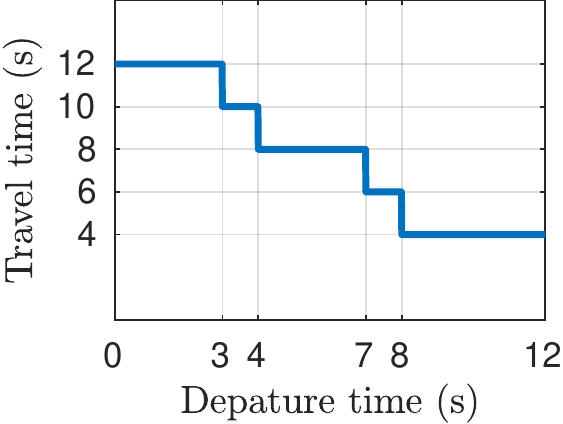} \label{fig:3x3_travel_s1}}
    \subfloat[Optimal travel time at~$s_2$]{\includegraphics[width=0.4\columnwidth]{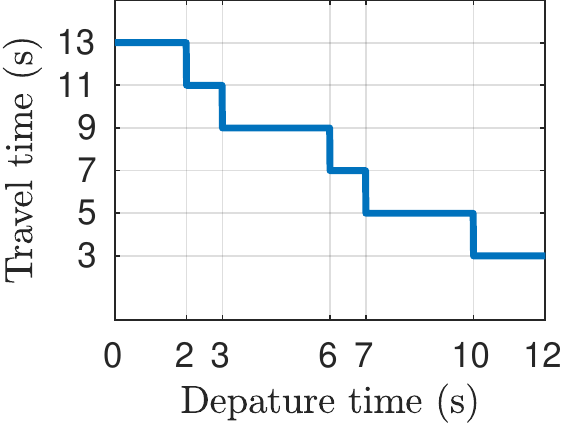} \label{fig:3x3_travel_s2}}
    \subfloat[Optimal travel time at~$s_5$]{\includegraphics[width=0.4\columnwidth]{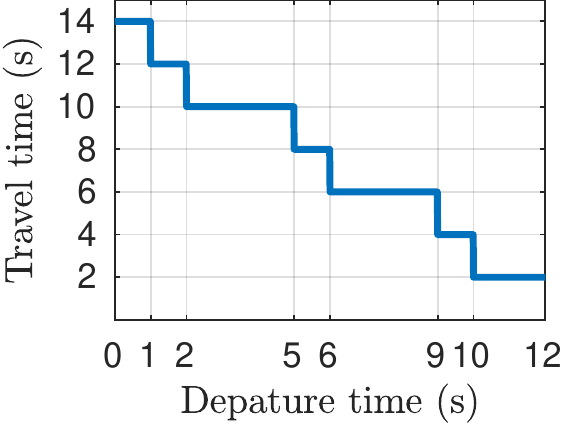} \label{fig:3x3_travel_s5}}
    \subfloat[Optimal travel policy at~$s_5$]{\includegraphics[width=0.4\columnwidth]{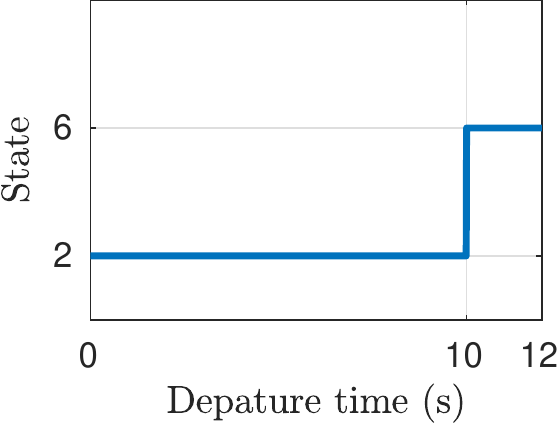} \label{fig:3x3_policy_s5}}
    
    \caption{3-by-3 example to reach~$s_9$. The red lines represent the optimal path starting from~$s_1$ to~$s_9$ and the width of blue lines represents the edge time for the corresponding edge where the thicker width illustrates shorter travel time. The current state is coloured in blue.}\label{fig:GridTimeStep}
\end{figure*}

\section{Optimal Planning Over Time-Dependent Flow} \label{sec:planning}

In this section, we propose a path planning framework for time-dependent flow fields. We present graph construction that guarantees asymptotic optimality. We then present how time-dependent edge time functions are derived for the graph.

\subsection{Time-dependent flow field and FIFO condition} \label{subsec:fifo_flow}


Suppose we have two points in a time-dependent flow field where a single control is found to traverse from one to the other.
The path at time~$t$ may be different to that at~$t' > t$. Since the path for departing later could yield shorter arrival time than departing earlier, time-dependent flow field does not respect FIFO.
As the distance between two points approaches zero, the non-FIFO property is weakened since the surrounding flow also approaches time-static constant flow. Therefore, time-dependent flow fields are asymptotically FIFO.


\subsection{Asymptotically optimal planning for TDSP problems} \label{subsec:AO_roadmap}

We use the \emph{probabilistic roadmap$^*$}~(\prm) algorithm to solve the TDSP problem by building a graph; this approach guarantees asymptotic optimality with respect to the number of states~\cite{star_Karaman2010}.
Partial construction of the PRM* graph is shown in Fig.~\ref{fig:prmstar}. The algorithm first samples~$N$ states in free space~$\chi_{free}$ and the states are connected to their neighbouring states within the connection radius~$r$. Each edge time is computed as shown in Sec.~\ref{subsec:flowfieldImp}.
The algorithm is guaranteed to provide asymptotic optimality when the radius~$r$ satisfies the following~\cite{star_Karaman2010}:
\begin{equation}\label{eqn:connectionRadius}
\begin{split}
    r > \gamma \cdot (\log{(N)}/N)^{1/d},
\end{split}
\end{equation}
where~$d$ is the dimension of free space~$\chi_{free}$.

Note that groups of states are connected in cycles where neighbouring regions overlap (i.e. states~$s_3$, $s_4$ and~$s_5$ in Fig.~\ref{fig:prmstar}). 
Therefore, the proposed algorithms can fully exploit cyclic edge connections.

\subsection{Time-dependent edge time functions from flow field}\label{subsec:flowfieldImp}

Suppose we have a vehicle~$R$ modelled as
$\dot{\mathbf{x}} = \mathbf{v}_R(\theta) + F(\mathbf{x}, t)$,
where~$\mathbf{x} \in \mathbb{R}^2$ is the position of the vehicle, $F(\mathbf{x}, t) = [u_{\mathbf{x}, t}, v_{\mathbf{x}, t}]^T$ is the time-dependent flow vector and $\mathbf{v}_R~\in \mathbb{R}^2$ is the vehicle velocity relative to the flow. The vehicle is controlled by varying the bearing angle~$\theta$ whilst travelling at constant speed~$V_{\max}$.
The discrete time model for the vehicle is represented as
\begin{equation} \label{eqn:discrete_model}
    \mathbf{x}[k+1] = \mathbf{x}[k] + (\mathbf{v}_R(\theta) + F(\mathbf{x}[k], t))\cdot \Delta t
    ,
\end{equation}
assuming the flow does not vary during~$\Delta t$.

Given two graph states~$s$ and~$s'$ that are located at~$\mathbf{x}_{s}$ and~$\mathbf{x}_{s'}$, respectively, we enumerate a set of control samples~$\Theta = \{\theta_0, \cdots\}$ (i.e., bearing angles) to compute the corresponding set of trajectories~$\mathbf{X} = \{\mathbf{x}^0, \cdots\}$ for a predefined time horizon~$h$. For a given control~$\theta_i$ and departure time~$t$, we start from~$\mathbf{x}^i[k] = \mathbf{x}_s$ until~$k = H$ using~(\ref{eqn:discrete_model}), where~$H = \text{ceil}(h / \Delta t)$ is discrete time horizon. Once the enumeration is completed, we find the trajectory~$\mathbf{x}^{i^*} \in \mathbf{X}$ that approaches closest to~$\mathbf{x}_{s'}$, such that
\begin{equation}
    i^* = \argmin_{i} \min_{k \leq H} \| \mathbf{x}_{s'} - \mathbf{x}^i[k] \|
    .
\end{equation}

The time taken for the trajectory to reach~$\mathbf{x}_{s'}$ is denoted as the edge time at departure time~$t$.
By Lemma~\ref{lemma:bound}, there exists a finite number of subdomains assuming that the flow forecast is also given in a form of piecewise-constant functions. 
This is a practically valid assumption as discussed in Sec.~\ref{sec:analysis}.

\begin{figure*}[t]
    \centering
    \subfloat[$t = 1.951$ departing at~$t_0 = 0$ (total travel time = $42.599$) ]{\includegraphics[width=0.9\columnwidth]{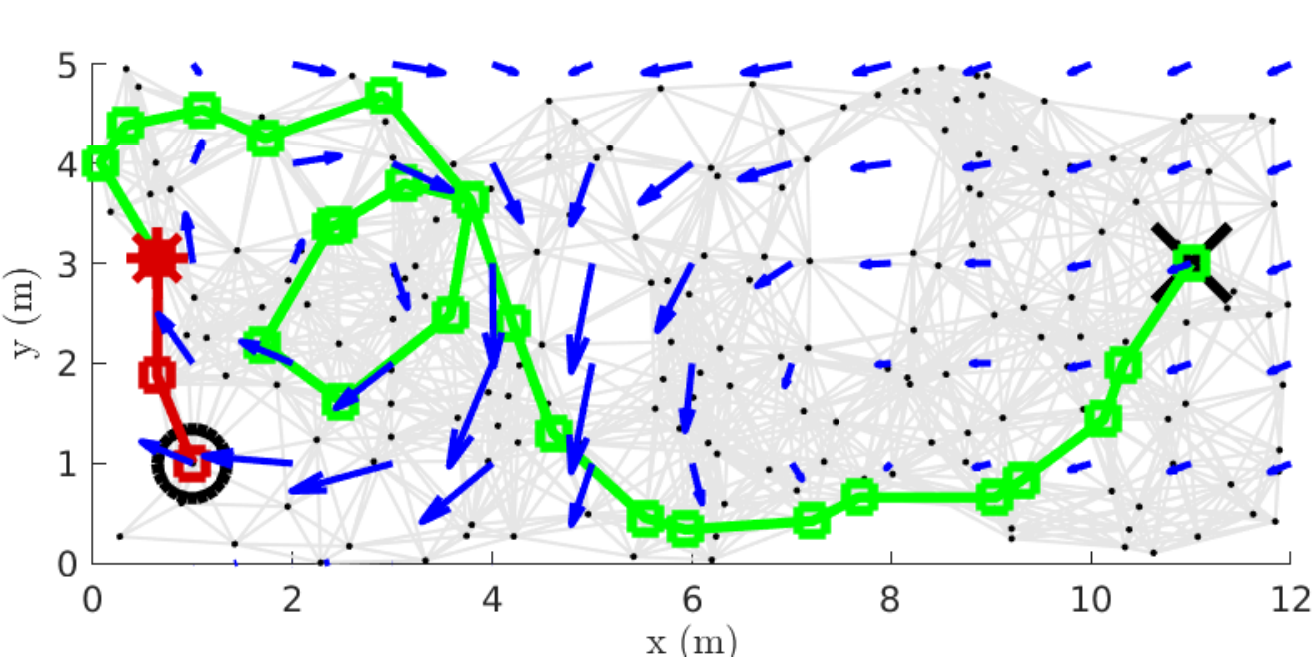} \label{subfig:flowfieldLargeCase_frame1}}
    \subfloat[$t = 10.916$ departing at~$t_0 = 0$ (total travel time = $42.599$)]{\includegraphics[width=0.9\columnwidth]{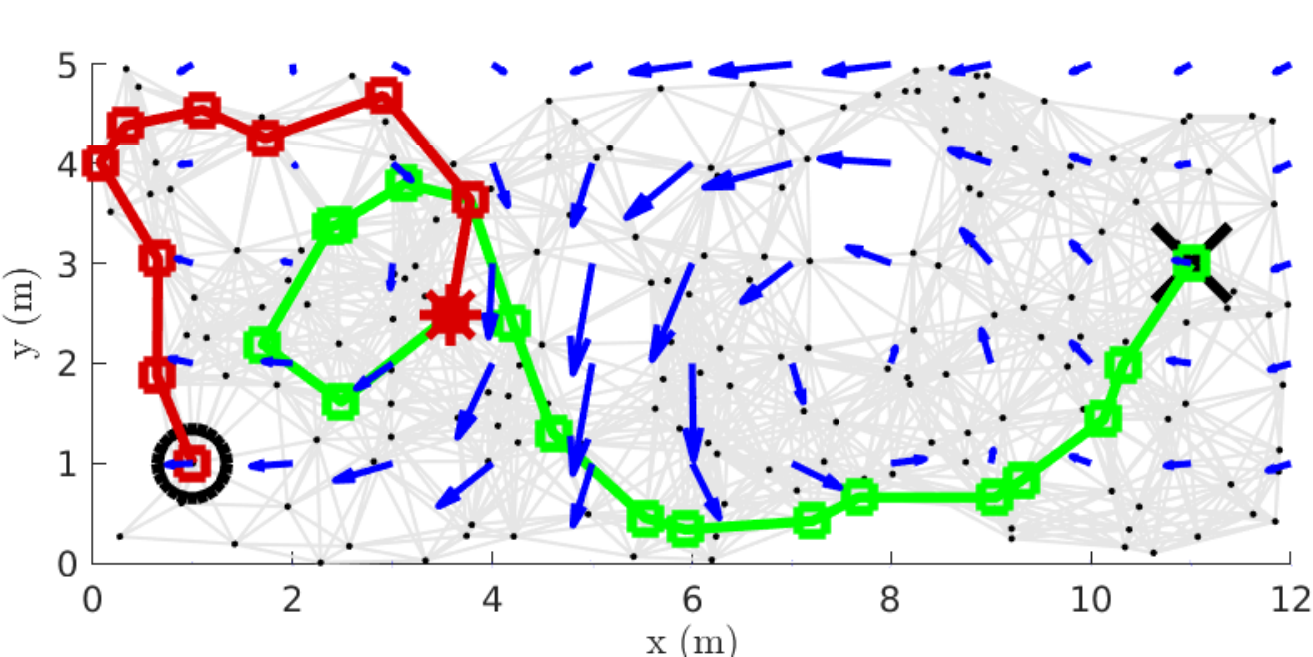} \label{subfig:flowfieldLargeCase_frame2}} \\
    \subfloat[$t = 29.684$ departing at~$t_0 = 0$ (total travel time = $42.599$)]{\includegraphics[width=0.9\columnwidth]{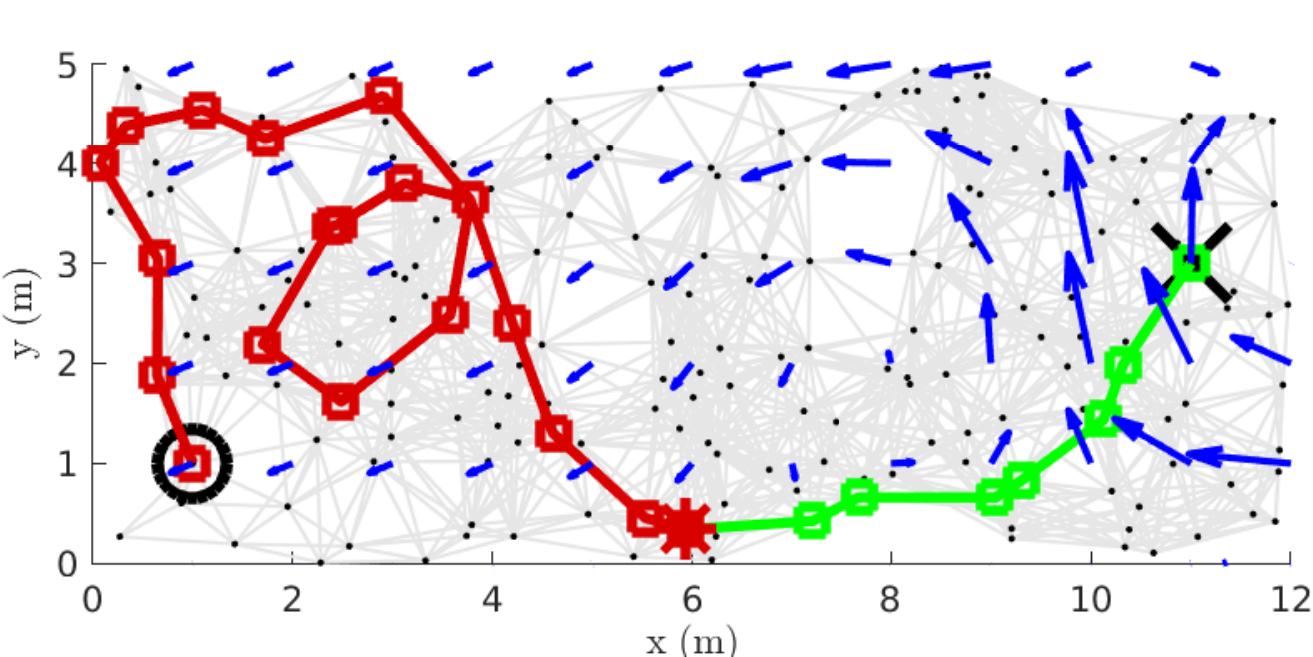} \label{subfig:flowfieldLargeCase_frame3}}
    \subfloat[$t = 42.599$ departing at~$t_0 = 0$ (total travel time = $42.599$)]{\includegraphics[width=0.9\columnwidth]{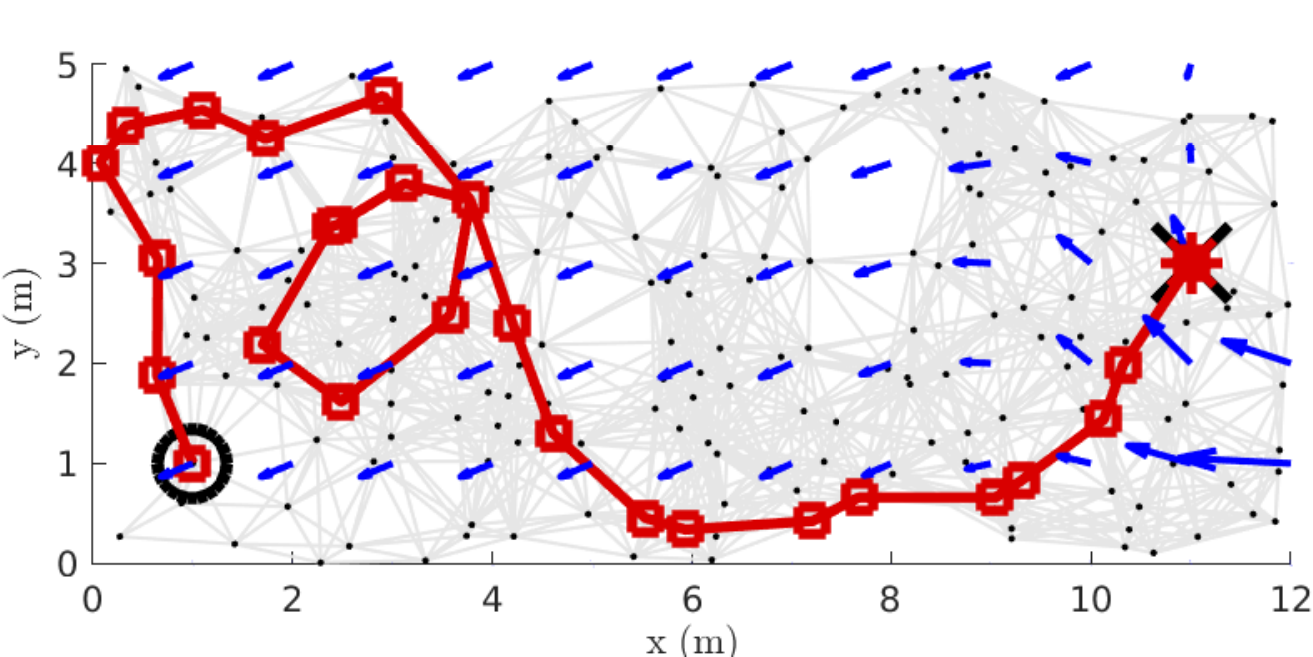} \label{subfig:flowfieldLargeCase_frame4}} \\
    \subfloat[$t = 31.8243$ departing at~$t_0 = 12$ (total travel time = $34.8365$)]{\includegraphics[width=0.9\columnwidth]{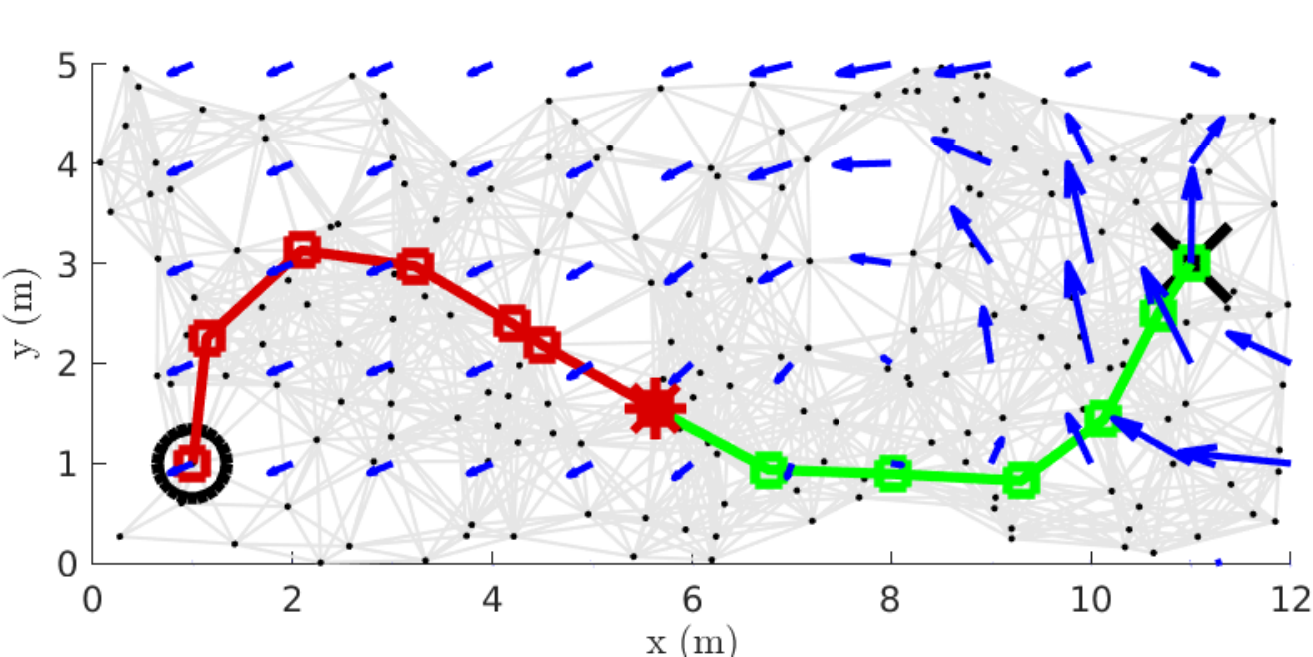} \label{subfig:flowfieldLargeCase_12}}
    \subfloat[$t = 40.4315$ departing at~$t_0 = 25$ (total travel time = $36.2631$)]{\includegraphics[width=0.9\columnwidth]{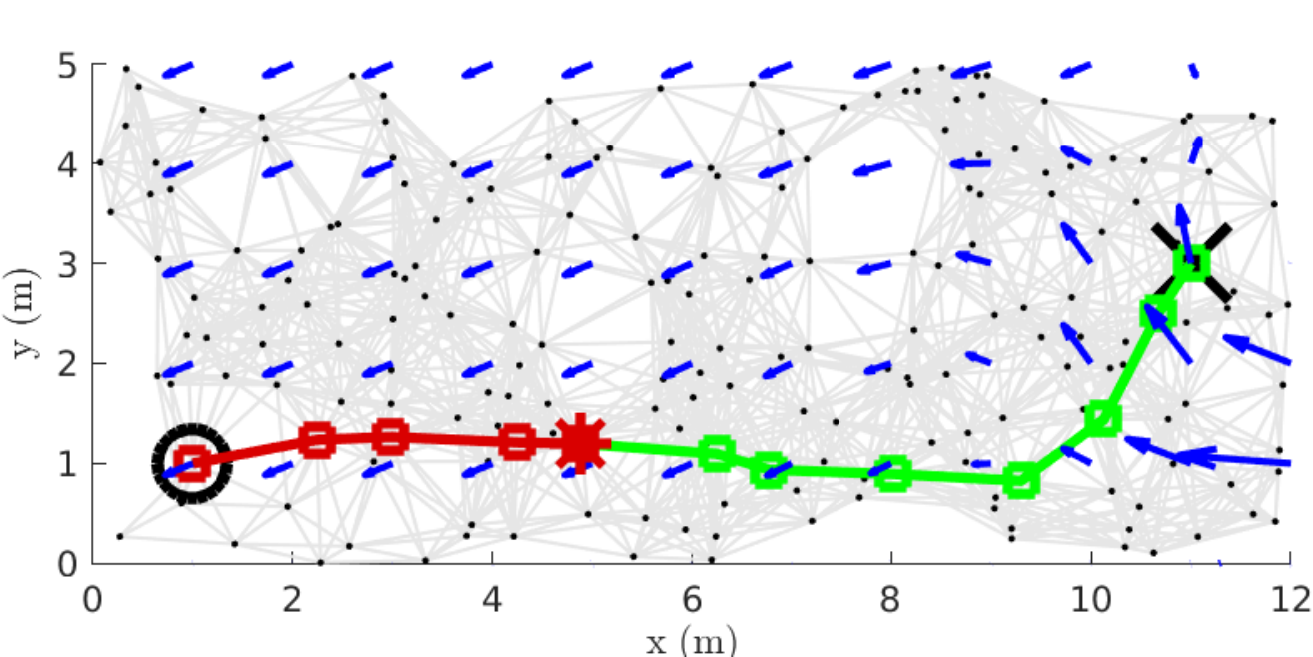} \label{subfig:flowfieldLargeCase_25}}
    \caption{Navigating at different initial departure time~$t_0 = 0$, $12$ and~$25$ through time-dependent flow field from start (bottom left) and to destination (top right). The red line is the path prefix and the green is the suffix, where~$t$ is the travel time for the prefix since the departure~$t_0$. The red asterisk is the current location. Blue arrows represent the flow vector at a given position. The \prm nodes and edges are shown in black.}
    \label{fig:ComplexFlowTimeStep}
\end{figure*}

\section{Examples}\label{sec:experiment}

In this section, we present two simulated examples where local travel time varies with departure time. 
We first demonstrate a discrete case with a discrete graph where we find an optimal policy and the corresponding path given initial departure time. Next, we present a flow field scenario where we sample a graph over a continuous flow field and solve for the path planning problem in an asymptotically optimal manner.
The algorithm was run on a standard laptop with Intel i5-6300 2.5GHz CPU and 8GB RAM.

\subsection{3-by-3 grid}

We consider a graph with 9 states as shown in Fig.~\ref{fig:GridTimeStep}, where we aim to reach state~$s_9$ from state~$s_1$. The width of each edge line (in blue) illustrates the corresponding edge time; the thickest edge width represents an edge time of~$1$ while that for the thinnest is~$25$. The current state is coloured blue.

In Fig.~\ref{fig:GridTimeStep}, the initial departure time is~$t_0 = 0$, when travel along edges in the direction of the goal state is  expensive (Fig.~\ref{subfig:GridCase_frame1}). The edge costs are relaxed later in time (Fig.~\ref{subfig:GridCase_frame4}). A policy derived using algorithms in ~Sec.~\ref{sec:synthesis} results in a path with optimal travel time~$\Gamma^* = s_1 s_4 s_5 s_2 s_1 s_4 s_5 s_2 s_1 s_4 s_5 s_6 s_9$,
which is drawn in red. 

The path starts off with two cycles through states~$\{s_1, s_2, s_4, s_5\}$ and then reaches the goal state via state~$s_6$. The optimal travel time for the path is~$12$ as shown in Fig.~\ref{fig:3x3_travel_s1}. State~$s_5$ plays an important role in the optimal policy, deciding when to move towards the goal state, while others' actions are greedily chosen.
The travel policy for states~$s_1$, $s_2$ and~$s_4$ is to move in a cycle regardless of departure time, whereas the policy for state~$s_5$ (Fig.~\ref{fig:3x3_policy_s5}) is to transit to~$s_6$ when the time is greater than~$10$.
The optimal solution was found after 14 iterations and the running time was~$0.418$ seconds to converge.


\subsection{Asymptotically optimal path planning in flow field}

We generated a time-dependent flow field using the Taylor-Green gyre vortex model~\cite{Taylor1937}, which is commonly used to model ocean currents as shown in Fig.~\ref{fig:ComplexFlowTimeStep} (in blue). As described in Sec.~\ref{sec:planning}, we randomly sampled 200 states over the space and used a connection radius of~$r = 1.735$ ~(\ref{eqn:connectionRadius}) to connect them. Sampled states and edges are shown in black.

The vehicle starts from the bottom left corner (circle) aiming to reach the top right corner (cross). Figures~\ref{subfig:GridCase_frame1}-\ref{subfig:GridCase_frame4} illustrate its progress over the optimal path for initial departure time~$t_0 = 0$ and the flow evolution over time. The red line represents the trajectory followed up to time~$t$ and the green line is the remaining path. The total travel time is~$42.599$. The vehicle spends time on the left side until the flow in the middle weakens. Then it moves towards the goal state against a weak opposing flow.
The optimal travel policy was found after 44 iterations and the running time was 583 seconds.

In Fig.~\ref{subfig:flowfieldLargeCase_12} and~\ref{subfig:flowfieldLargeCase_25}, we demonstrate optimal paths for two different initial departure times,~$t_0 = 12$ and~$25$, respectively. Since the flows are more relaxed than departing at~$t_0 = 0$, the overall travel times are reduced although the arrival times are later. 
It is important to note that optimising travel time for ocean vehicles is more valuable than optimising arrival time when their endurance is dominated by travel time, as is often the case. We have shown that our proposed framework can find the optimal departure time that minimises the overall travel time in one such example.



\section{Conclusion and Future Work}\label{sec:conclusion}
We have presented a new algorithm for finding shortest paths in time-dependent graphs and have shown how it can be combined with sampling-based methods for planning in time-dependent flow fields. This result expands the set of known polynomial-time special cases for TDSP to include non-FIFO problems with piecewise-constant edge travel time functions. Previously, special cases with restricted slopes were restricted to FIFO problems. Our result also provides a new practical solution for planning in flow fields, with known performance bounds, that is applicable to autonomous vehicles in the ocean~\cite{cadmus2019streamlines,lee2018active}. An important avenue of future work is to develop efficient algorithms with performance guarantees for stochastic cases where flow field velocities and vehicle control are uncertain. It is also important to explore further special cases in non-FIFO TDSP problems that may be solved in polynomial time.

\section*{Acknowledgement}
This work is supported by an Australian Government Research Training Program (RTP) Scholarship, Australia's Defence Science and Technology Group and the University of Technology Sydney.

\bibliographystyle{abbrvnat}
\bibliography{references}

\end{document}